\documentclass[a4paper,twocolumn,11pt]{quantumarticle}
\pdfoutput=1
\usepackage[utf8]{inputenc}
\usepackage[english]{babel}
\usepackage[T1]{fontenc}
\usepackage{amsmath}
\usepackage{amssymb}
\usepackage{amsfonts} %
\usepackage{xurl}
\usepackage{hyperref}
\usepackage{cleveref}
\usepackage{aliascnt}
\usepackage[numbers,sort&compress]{natbib}
\usepackage{enumitem}
\usepackage{braket}
\usepackage[ruled]{algorithm2e}
\usepackage{algorithmic}

\usepackage{tikz}
\usepackage{lipsum}
\usepackage{amsthm}
\usepackage{float}

\crefname{algocf}{alg.}{algs.}
\Crefname{algocf}{Algorithm}{Algorithms}

\newcommand{\E}{{\mathbb E}}

\newcommand{\RR}{{\mathbb R}}

\newcommand{\ep}{\epsilon}

\newcommand{\tr}{\mathrm{Tr}}

\newcommand{\bi}{\begin{itemize}}
\newcommand{\ei}{\end{itemize}}

\newcommand{\beq}{\begin{equation}}
\newcommand{\eeq}{\end{equation}}
\newcommand{\beqa}{\begin{eqnarray*}}
\newcommand{\eeqa}{\end{eqnarray*}}
\newcommand{\btm}{\begin{theorem}}
\newcommand{\etm}{\end{theorem}}
\newcommand{\bpf}{\begin{proof}}
\newcommand{\epf}{\end{proof}}
\newcommand{\bla}{\begin{lemma}}
\newcommand{\ela}{\end{lemma}}
\newcommand{\bdn}{\begin{definition}}
\newcommand{\edn}{\end{definition}}
\newcommand{\bpn}{\begin{proposition}}
\newcommand{\epn}{\end{proposition}}
\newcommand{\bcy}{\begin{corollary}}
\newcommand{\ecy}{\end{corollary}}
\newcommand{\kl}{\mathrm{kl}}

\def\ldotsplus{\mathinner{\ldotp\ldotp\ldotp\ldotp}}
\def\fourdots{\relax\ifmmode\ldotsplus\else$\m@th \ldotsplus\,$\fi}

\newtheorem{theorem}{Theorem}[section]

\newaliascnt{lemma}{theorem}
\newtheorem{lemma}[lemma]{Lemma}
\aliascntresetthe{lemma}

\newaliascnt{proposition}{theorem}
\newtheorem{proposition}[proposition]{Proposition}
\aliascntresetthe{proposition}

\newaliascnt{corollary}{theorem}
\newtheorem{corollary}[corollary]{Corollary}
\aliascntresetthe{corollary}

\newaliascnt{definition}{theorem}
\newtheorem{definition}[definition]{Definition}
\aliascntresetthe{definition}

\newaliascnt{remark}{theorem}

\aliascntresetthe{remark}

\newaliascnt{example}{theorem}

\aliascntresetthe{example}

\crefname{theorem}{theorem}{theorems}
\Crefname{theorem}{Theorem}{Theorems}

\crefname{lemma}{lemma}{lemmas}
\Crefname{lemma}{Lemma}{Lemmas}

\crefname{proposition}{proposition}{propositions}
\Crefname{proposition}{Proposition}{Propositions}

\crefname{corollary}{corollary}{corollaries}
\Crefname{corollary}{Corollary}{Corollaries}

\crefname{definition}{definition}{definitions}
\Crefname{definition}{Definition}{Definitions}

\crefname{remark}{remark}{remarks}
\Crefname{remark}{Remark}{Remarks}

\crefname{example}{example}{examples}
\Crefname{example}{Example}{Examples}

\begin{document}

\title{Variational Quantum Optimization with Continuous Bandits}

\author{Marc Wanner}
\affiliation{Chalmers University of Technology and University of Gothenburg}
\email{wanner@chalmers.se}

\author{Johan Jonasson}
\affiliation{Chalmers University of Technology and University of Gothenburg}
\email{jonasson@chalmers.se}
\author{Emil Carlsson}
\affiliation{Sleep Cycle AB}
\email{emil.carlsson@sleepcycle.com}
\author{Devdatt Dubhashi}
\affiliation{Chalmers University of Technology and University of Gothenburg}
\email{dubhashi@chalmers.se}

\maketitle

\begin{abstract}
We introduce a novel approach to variational Quantum algorithms (VQA) via continuous bandits. VQA are a class of hybrid Quantum-classical algorithms where the parameters of Quantum circuits are optimized by classical algorithms. Previous work has used zero and first order gradient based methods, however such algorithms suffer from the barren plateau (BP) problem where gradients and loss differences are exponentially small. We introduce an approach using bandits methods which combine global exploration with local exploitation. We show how VQA can be formulated as a best arm identification problem in a continuous space of arms with Lipschitz smoothness. While regret minimization has been addressed in this setting, existing methods for pure exploration only cover discrete spaces. We give the first results for pure exploration in a continuous setting and derive a fixed-confidence, information-theoretic, instance specific lower bound. Under certain assumptions on the expected payoff, we derive a simple algorithm, which is near-optimal with respect to our lower bound. Finally, we apply our continuous bandit algorithm to two VQA schemes: a PQC and a QAOA quantum circuit, showing that we outperform  or are competitive with  state of the art methods based on finite difference schemes.
\end{abstract}

\section{Introduction}

In recent years, \emph{variational quantum computing} has gathered momentum as a promising approach for quantum computers~\cite{Abb24,Cerezo2021}, namely a \emph{hybrid} classical-quantum framework which involves a quantum circuit with gates parameterized by continuous real variables (see Figure~\ref{fig:vqc}). Potential application areas range from Quantum chemistry~\cite{Cao2019}, drug discovery~\cite{Blunt2022} and material science~\cite{Lordi2021} to Finance~\cite{Herman2023}, supply chain management and manufacturing~\cite{refId0}, where the Quantum circuit is used as an accelerator for specific domain-dependent problems. The circuit is specified by a set of real valued parameters which are tuned iteratively to optimal values by observing the circuit output using classical optimization algorithms. This approach has become popular in part to its flexibility, opening up applications in diverse areas in basic science and machine learning, and also because of the hope that it is more robust to the constraints of near-term quantum hardware in the NISQ (Noisy Intermediate Scale Quantum) era.

However, a big challenge for VQAs is that gradient based zero or first-order methods, such as COBYLA~\cite{Powell1994}, become stuck because of the landscape of the optimization problem with increasing problem size: gradients or even loss differences become exponentially small which makes it difficult to identify local descent directions as the system size increases - exponentially many samples are provably necessary to identify descent directions. This is called the \emph{barren plateau} (BP) phenomenon in Quantum computing~\cite{La24} and affects various VQA, such as the Quantum Alternate Operator Ansatz~\cite{hadfield2019from, Fontana2024}. \\
\begin{figure}
    \centering
    \includegraphics[width=0.5\textwidth]{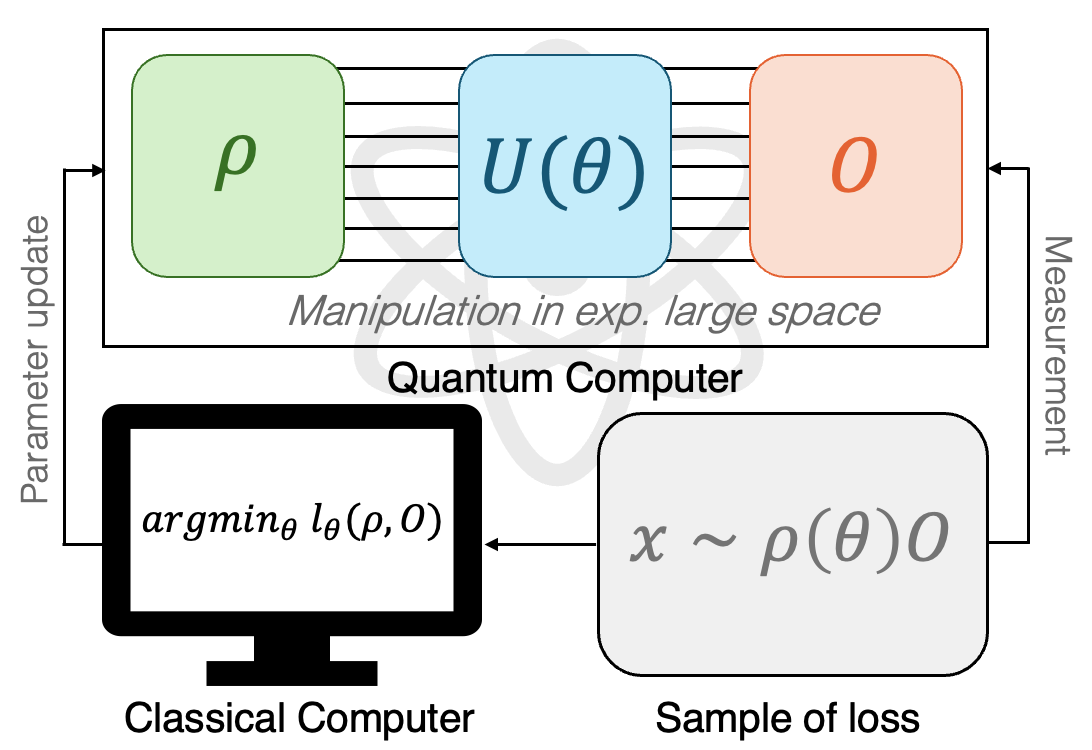}
    \caption{Illustration of VQA inspired by~\cite{La24}}
    \label{fig:vqc}
\end{figure}
While significant effort has been dedicated to gradient based methods and trying to construct VQAs avoiding BPs (see~\Cref{sec:related_work}), addressing the BP problem with other classical optimization techniques has received limited attention: Gradients are viewed as an indispensable component of VQAs, like they are for training classical neural networks. However, the evaluation of gradients or even the objective of a VQA  may require many circuit evaluations, indicating that optimization based on individual VQA ``shots'', instead of the latter quantities, may be more natural.\\
Here we introduce a novel approach to address the BP problem which is theoretically well grounded and also easy to implement in practice. We argue that the optimization problem is very well suited to bandit methods which combine global exploration and local exploitation. The optimal setting of the circuit parameters can be regarded as an optimization problem in a continuous bandit. \\
Black box optimization of noisy cost functions is a widely studied area with many applications \cite{bouneffouf2019surveypracticalapplicationsmultiarmed,Cerezo2021,maghsudi2016multi}. One approach to this problem is via \emph{bandit} algorithms, see ~\cite{lattimore2020bandit} for a textbook treatment.
In this work, we consider the best arm identification problem in bandits ~\cite{audibert:hal-00654404}. Formally, the typical problem setup consists of a set of distributions with means $\{\mu_a\}_{a=1}^A$, where each of them corresponds to an ``arm'' $a$. An agent is then tasked to identify the maximizing parameter, or ``arm'' $a^*$, with a given confidence by drawing a minimal amount of samples from the distributions.\\
For finite set of arms, this problem is well understood and can be solved optimally by e.g. \emph{Track-and-Stop}~\cite{pmlr-v49-garivier16a} or game theoretic exploration schemes~\cite{NEURIPS2019_8d1de745}. Further extensions provide optimal sampling strategies in settings where the arms reveal contextual information~\cite{wang2021fast} or are subject to linear constraints~\cite{pmlr-v238-carlsson24a}. 
An important property of these algorithms is their optimality with respect to the specific problem instance, which fundamentally differs from worst case optimality.\\
We consider the extension of the problem to the \emph{continuous} setting, where the set of arms corresponds to the unit interval $[0,1]$: thus there are uncountably many arms, each corresponding to a point in the unit interval: for every $x \in [0,1]$, we have a distribution with mean $\mu(x)$. We assume the function $\mu: [0,1] \rightarrow \mathbb{R}$, to be Lipschitz. Furthermore, the noise of the drawn samples is assumed to be sub-Gaussian. The task of the agent is to identify an arm $x$, which is $\epsilon$ close to the optimal arm $x^*$ with high probability. The classical optimization problem in variational quantum algorithms corresponds exactly to this best arm identification problem in continuous bandit setting.\\
Continuous bandit optimization was addressed in previous work. Two very well known examples are the X Armed Bandits\cite{BubeckMSS11} and the Zooming algorithm \cite{KleinbergSU19}. These algorithms addressed the \emph{regret} version of the bandit problem which is known to differ in important ways from the best arm identification problem. The former aims to accumulate as few sub-optimal samples as possible on an infinite time horizon, while the latter seeks to identify a probably approximatly correct (PAC) optimum under a minimal amount of samples. Here we address the best arm identification problem in continuous settings, which we consider a better representation for training VQA. Our bandit methods that combine global and local optimization are also applicable to better VQA designs and could be combined into hybrid schemes with purely local gradient based methods.\\
Specifically, we derive an instance specific lower bound for continuous, Lipschitz bandits on $[0,1]$. Furthermore, we present a simple algorithm, whose sample complexity matches the lower bound up a log factor. Our experiments show an improvement in scaling over not instance specific methods and, due to its simplicity, substantially lower cost per iteration with respect to non-adaptive, instance optimal methods. Hence, our algorithm has a practical advantage over these methods, which turn out to be intractable when exceeding a certain number of arms.
Our main contributions are:
\begin{itemize}
    \item We introduce a novel approach to hybrid Quantum-classical algorithms such as parametrized quantum algorithms (PQC) algorithms and QAOA by formulating it as a best arm identification problem in continuous bandits.
    \item We give an information-theoretic, instance specific lower bound for continuous best arm identification (in 1 dimension).
    \item We give an algorithm via adaptive partitioning that essentially meets the lower bound (up to a log factor) and linear computational complexity with respect to the number of samples.
    \item We give an algorithm that serves as a proxy for the multi-dimensional extension of the continuous bandit.
    \item We apply our bandit algorithm to PQC and QAOA problem instances, showing that it outperforms or is competitive with previous (finite difference) methods from literature. We also show synthetic examples where our method works while previous class of methods including SPSA fail catastrophically.
\end{itemize}

\section{Related Work}\label{sec:related_work}
\paragraph*{Structured and Continuous Bandits}
Best arm identification for general bandit problems with a finite set of arms has been explored in both worst-case scenarios~\cite{10.1007/978-3-642-04414-4_7} and instance-specific settings~\cite{pmlr-v49-garivier16a, NEURIPS2019_8d1de745}. Subsequent research has extended this work to provide instance-specific bounds and algorithms for bandits incorporating contextual information, such as those with Linear or Lipschitz-continuous reward functions~\cite{wang2021fast}, as well as bandits subject to constraints~\cite{pmlr-v238-carlsson24a}. Continuous bandits have primarily been studied in the context of regret minimization~\cite{bubeck2011x, KleinbergSU19}, with further refinement in Adaptive-treed bandits~\cite{bull2015adaptive}, which, while focusing on cumulative regret, also offers PAC bounds. In the context of Quantum computing, regret minimization has been studied outside of VQA, where the task is to select an optimal observable from a finite or continuous set~\cite{Lumbreras2022multiarmedquantum, lumbreras2024learningpurequantumstates}, with regret defined as the expected measurement outcome, though the bounds in this setting are not instance-specific. Methods for pure exploration in continuous bandit settings, such as \emph{MFDOO}~\cite{demontbrun2024certified}, have been developed; however, they do not incorporate structural properties of the problem. Additionally, these methods rely on worst-case analysis and lack instance-dependent performance guarantees. There are no known bounds for continuous bandits with Lipschitz reward functions. Our work is the first to address this gap.
\paragraph*{Mitigating BP in VQAs}
Recent research indicates a trade-off between the expressivity and trainability~\cite{Holmes_2022} of specific quantum circuit architectures, often referred to as \emph{ansatz}. If an ansatz lacks sufficient expressivity, it may be incapable of representing the target function. Conversely, provably expressive ansatzes are typically susceptible to the BP phenomenon, which complicates the task of identifying the desired model within the represented model class. 
\cite{La24} provides a summary of current techniques for mitigating BPs, a subset of which we briefly outline. One way to mitigate this issue is to find the best trade-off via adaptive structure search~\cite{Du2022} and \emph{ADAPT-VQE}~\cite{Grimsley2023}. Furthermore, most proofs of presence of BPs only apply to random parameter initialization. Therefore, employing alternative initialization strategies~\cite{NEURIPS2022_7611a3cb} can serve as an effective approach to mitigating this issue. Finally, certain architectures, such as noise-induced shallow circuits~\cite{mele2024noiseinducedshallowcircuitsabsence}, are proven not to have BPs. However, recent work suggests that the absence of BPs implies classical simulability~\cite{cerezo2024doesprovableabsencebarren}. Even in the absence of BPs, the challenge of avoiding local minima remains. Alternative training methods for VQAs, such as those proposed in~\cite{PhysRevX.7.021027, PhysRevA.107.032407, agirre2024montecarlotreesearch,NFT}, lack formal theoretical guarantees, weakening evidence of strong performance on larger problem instances. The same holds for applications of Bayesian optimization~\cite{nicoli2024physicsinformedbayesianoptimizationvariational}, which come with substantial computational complexity and potentially hard optimization of an acquisition function after each circuit evaluation. In this work, we introduce the application of frequentist bandit methods to this domain for the first time.

\section{Preliminaries}\label{sec:preliminaries}
This section introduces the general \emph{stochastic bandit} model and its connection to VQAs. 
\subsection{Continuous Bandits}\label{sec:continuous_bandits}
In its general form, a \emph{stochastic bandit} is a pair $(\mathcal{X}, M)$ of a measurable space $\mathcal{X}$ and a set of random variables $M$. Each element $x \in \mathcal{X}$ is associated with a random variable $M(x) \in M$. Playing arm $x\in \mathcal{X}$ corresponds to observing an i.i.d. sample from $M(x)$. In this work, we consider \emph{continuous bandits}, i.e., bandits with uncountably infinite set of arms $\mathcal{X} \subset \RR^d$, whose expected rewards are denoted by $\mu(x) = \E[M(x)]$. Typically, $\mu(x)$ is assumed to be $L$-Lipschitz and the rewards $M(x)$ to be either sub-Gaussian or restricted to a bounded domain~\cite{bubeck2011x}. Stochastic bandits describe a setting where the only way of accessing $\mu$ is to observe samples of the respective rewards. Therefore, the continuous bandit model is directly applicable to VQA, which is outlined in the upcoming section.

\subsection{Variational Quantum Optimization}
\label{sec:bp}
As illustrated in~\Cref{fig:vqc}, in variational quantum computing, the process begins by initializing the quantum system to an $n$-qubit state $\rho$, which is then passed through a \emph{parameterized quantum circuit} (PQC). Since all quantum circuits consist of unitary operations, a PQC can be expressed as a sequence of $p$ parametrized unitaries $U(x) = \prod_{i=1}^p U_i(x_i)$,
where $x = (x_1,...,x_L)$ is a set of trainable parameters. After applying $U(x)$, the resulting state is measured with an observable $O$, yielding an outcome $o_z$, an eigenvalue of $O$ associated with eigenvector $\ket{z}$. By the \emph{Born rule}, the probability of measuring $o_z$ is given by $\bra{z} \rho(x) \ket{z}$, with $\rho(x) = U(x) \rho U(x)^\dagger$. The expected outcome is given by $\tr[\rho(x) O]$, which leads to the loss function 
$\ell_{x}(\rho, O) = \tr[\rho(x) O]$. The expected loss is typically approximated by repeatedly running the circuit with the same parameters and computing the empirical average.\\ 
Reformulating this setup as measuring the initial state (possibly unknown) $\rho$ with an observable from the set $\{\tilde{O}(x)\}_{x}$ with $\tilde{O}(x) = U^\dagger(x) O U(x)$ can be described as \emph{multi-armed Quantum Bandits}~\cite{Lumbreras2022multiarmedquantum}. Since the initial state $\rho$ of a given PQC is known, optimizing a PQC reduces to a classical, continuous bandit problem. Explicitly, $\mu(x) = \tr[\rho(x) O]$ and 
\begin{equation}
    P(M(x) = y) = \begin{cases}
        \bra{z} \rho(x) \ket{z}  & \text{if } y = o_z,\\
        0 &\text{otherwise}.
    \end{cases}
\end{equation}
If the eigenvalues of $O$ lie between $0$ and $1$, so does the reward distribution. This is easily achieved by linearly transforming $O$. Furthermore, general Unitaries $U(x_k)$ can be expressed as $e^{-ix_k V_k}$, where $V_k$ is Hermitian with bounded spectral norm~\cite{Holmes_2022}. Therefore, $\mu$'s gradients are bounded by some constant $L$. Although our results only apply to $1$-d continuous bandits, they can be used for multidimensional continuous bandits, for example in combination with Powell's method~\cite{powell1964efficient}, which we will discuss in more detail in~\Cref{sec:algorithm}.

\subsection{Problem statement}

In~\Cref{sec:lower_bound} and~\Cref{sec:algorithm}, we consider a bandit problem $([0,1], M(x))$, which satisfies the following assumptions:

\begin{enumerate}[label=(\roman*)]
    \item The rewards $M(x)$ are $1$-sub-Gaussian.
    \item The expected reward $\mu(x)$ is $L$-Lipschitz.
    \item $\mu(x)$ has a unique optimum $x^*$. \label{ass:3}
    \item There is a constant $\kappa_0$, such that $\mu$ is unimodal on every set $E \subseteq [x_i, x_i + \kappa]$ for all $\kappa < \kappa_0$. \label{ass:4}
\end{enumerate}

The optimization algorithm aims to find an arm, which is close to the unique maximum\\
$x^* = \mathrm{argmax}_{x\in [0,1]} \mu(x)$ with high probability. Formally, we want the learner to recommend an arm $\hat{x}$, such that, given $\ep, \delta > 0$,
\begin{equation}\label{eq:pac_learner}
    \Pr(\lVert \hat{x} - x^* \rVert_2 > \ep) \leq \delta.
\end{equation}
An optimization algorithm that satisfies~\Cref{eq:pac_learner} is often referred to as $\delta$-\emph{PAC learner}. Our goal is to design a PAC-learner, which finds an $\ep$-optimal arm with probability at least $1-\delta$ after a minimal number of steps, which we refer to by  $\tau_{\delta}^{\ep}$. In general, this stopping time is random so that we quantify the sample complexity of the algorithm by $\E[\tau_{\delta}^{\ep}]$.
Note that assumptions~\ref{ass:3} and~\ref{ass:4} do not hold for general PQCs. In practice, this is not an issue, as these assumptions are not needed for our algorithm to find an arm with reward $\mu(x) \leq \mu(x^*) + \ep$.
\paragraph*{Notation}
In the following, we introduce some notation, which will occur throughout~\Cref{sec:lower_bound} and~\Cref{sec:algorithm}. All other definitions will be introduced in the respective subsections. In the upcoming sections, we will consider the equivalent minimization problem for convenience. Therefore, let $v(x) \triangleq \mu(x^*) - \mu(x)$, such that $v$ has its minimum where $\mu$ has its maximum and $v(x^*) = 0$. For clarity, we sometimes write $x^*(f) \triangleq \mathrm{argmax}_{x \in [0,1]} f(x)$ in order to refer to the maximizer of a specific function, which we omit for $\mu$ and $v$.\\
\begin{figure}
    \centering
    \includegraphics[width=0.5\textwidth]{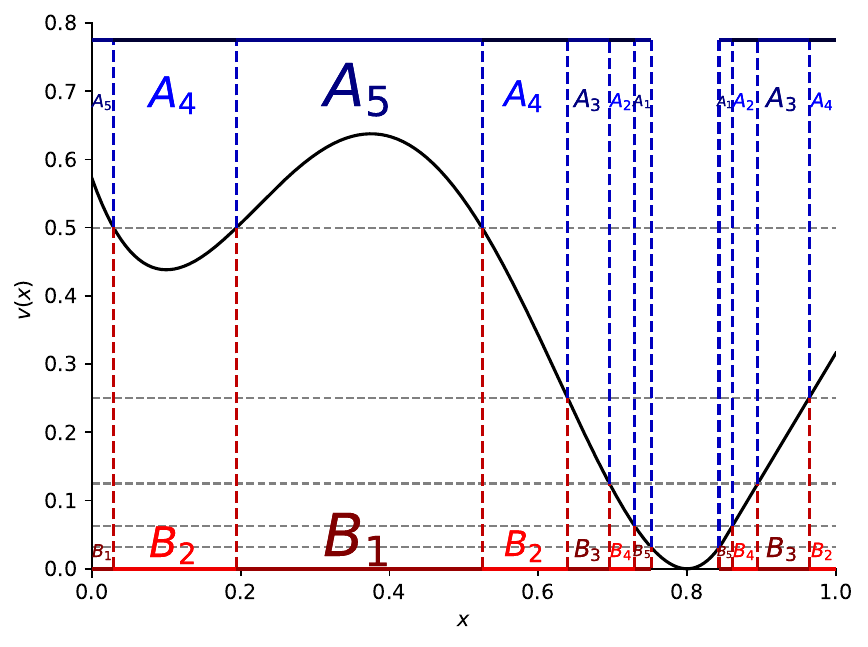}
    \caption{Example with $\ep=2^{-5}, \kappa_0=\frac{1}{4},S=3$.}
    \label{fig:sketch_sets}
\end{figure}
For convenience, we assume that $\epsilon = 2^{-D}$. Furthermore, we define the following level sets, i.e. sets of the form 
\begin{equation*}
    v^{-1}[a,b] \triangleq \{x \in [0,1] : a \leq v(x) \leq b\}.
\end{equation*}
Furthermore, we define $A_t = v^{-1}(2^{t-1}\epsilon, 2^t \epsilon]$ and \\
$B_t = v^{-1}(2^{-t}, 2^{-(t-1)}]$ for $t=1,\ldots,D$. Note that $A_t = B_{D-t}$. This is illustrated in \Cref{fig:sketch_sets}. We use the Lebesgue measure, which we denote by $m(\cdot)$, to express the ``length'' of the sets. \\ 
Finally, we introduce two additional definitions.
\begin{definition}[Covering-number]
    The \emph{covering number} $N_r(X)$ is the smallest number of sets with diameter $r$ required to cover a set $X$.
\end{definition}

The covering number gives rise to the \emph{zooming dimension}, a concept, which is used in related literature, such as~\cite{KleinbergSU19}.

\begin{definition}[Zooming-dimension]
    The \emph{zooming dimension} of an instance $\mu(x)$ is defined as the smallest $\beta$, for which there is a constant $C$, such that $N_{r/8}(X_r) = Cr^{-\beta}$ for every $1/2 > r>0$ and sets $X_r \triangleq v^{-1}(r, 2r]$.
\end{definition}

The zooming dimension is a convenient tool to write regret upper and lower bounds in concise form. Conceptually, it can be thought of as the limit of the ``flatness'' of $\mu$ around its maximum. In the $1$-d setting, $\beta$ takes values from $0$ to $1$.

\section{Lower bound}\label{sec:lower_bound}

In this section, we first establish an information-theoretic lower bound, following a similar approach to that of ~\cite{pmlr-v49-garivier16a, pmlr-v238-carlsson24a}. Deriving a scheme analogous to \emph{Track-and-Stop} directly from this bound is infeasible, as explicitly solving for $w$ is intractable. Therefore, we derive a simplified, instance dependent lower bound. The proofs of all theoretical results in this section are provided in~\Cref{sec:proofs_lb}.

\subsection{Continuous lower bound}

The methods to obtain a lower bound in the discrete, unstructured setting~\cite{pmlr-v49-garivier16a} can easily be adapted to the infinite-arm setting we consider. The main differences lie in the definition of the alternate set and the use of an integral instead of a sum, leading to a lower bound that closely resembles the one derived in the discrete setting. Denote the \emph{alternate set} by \begin{equation}
    \mathrm{Alt}^{\ep}(\mu) \triangleq \{ \lambda \ L\text{-Lipschitz}: \lVert x^*(\mu) - x^*(\lambda) \rVert > \ep \}.
\end{equation}
For convenience, we only consider sample strategies, which have Riemann-integrable probability density and denote the respective set by $\mathcal{W}$. In the following, we present an information-theoretic lower bound for continuous bandits.

\begin{theorem}\label{thm:exact_lower_bound}
    Let $\mu$ be a bandit continuous bandit on some set $\mathcal{X}$.
    For any $(\ep, \delta)$-PAC learner,  
    \begin{equation}
        \E[\tau_{\delta}^{\ep}] \geq \frac{\log(1/\delta)}{c^*(\mu)},
    \end{equation}
    where 
    \begin{equation}\label{eq:c_star_main}
        c^*(\mu) = \sup_{w \in \mathcal{W}} \inf_{\lambda \in \mathrm{Alt}^{\ep}(\mu)} \int_{\mathcal{X}} w(x) (\mu(x) - \lambda(x))^2 dx.
    \end{equation}
\end{theorem}

The proof of~\Cref{thm:exact_lower_bound} can be extended to alternative definitions of $\mathrm{Alt}^{\ep}(\mu)$ and holds for a general domain $\mathcal{X}$. It is a consequence of information theoretic properties of a PAC learner. Analogously to its discrete counterpart~\cite{NEURIPS2019_8d1de745},~\Cref{eq:c_star_main} has a game-theoretic interpretation: the $w$-player maximizes the value in~\Cref{eq:c_star_main}, to which the $\lambda$-player responds with $\lambda$ that minimizes the objective for the given $w$. This perspective will allow us to derive a more tractable lower bound, as we will see in~\Cref{sec:discrete_lower_bound}.
Our choice of alternate set, i.e. restricting the location of the maximum of the confusing instance $\lambda$ seems like the most natural extension of its discrete counterpart. However, in some cases, the objective may be to identify $x$, such that $\mu(x^*) - \mu(x) \leq \ep$. Consequently, one may consider alternate instances where $\max_x \mu(x) - \max_x \lambda(x) > \ep$. Though, this choice of alternate set does not properly reflect the optimization task and results in a strictly weaker, instance-independent lower bound.

\begin{corollary}\label{cor:trivial_lower_bound}
Let $\mathcal{L}$ be the space of $1$-Lipschitz functions on $\mathcal{X}$ and
\begin{equation}
    \mathrm{Alt}^{\ep}(\mu) \triangleq \{\lambda \in \mathcal{L}:\max_x \mu(x) - \max_x \lambda(x) > \ep\}.
\end{equation} Then, 
    \begin{equation}
        c^*(\mu) = \ep^{2}.
    \end{equation}
\end{corollary}

This result supports our previous definition of the alternate set.

\subsection{Simplified lower bound}\label{sec:discrete_lower_bound}

In the following, we present an upper bound for $c^*(\mu) = c^*(v)$,
which serves as a lower bound for $\E[\tau_{\delta}^{\ep}]$. We consider minimizing $v$ for convenience and refer to $v'(x) = \lambda(x^*(\mu))-\lambda(x)$ as the corresponding quantity for the alternate instance, denoted by $\lambda$ in \Cref{eq:c_star_main}.
One can interpret $c^*(v)$ as a game, where two players compete against each other in alternating rounds: one player tries to minimize the expression by assigning values to $v'$, while the other one tries to maximize it by applying values to $w$. 
In the beginning, the $w$-player may try to assign $v'$, such that it equals $v$ everywhere except in a small region around the minimum, ensuring that $v'\in \mathrm{Alt}^\ep(v)$. The $w$-player's approximately best response is to assign $w \equiv 0$ where $v'=v$ and uniform where $v' \neq v$. The $v'$-player may now be unable to play any $v'\in \mathrm{Alt}(v)$ which increases the objective, implying that the objective is close to $c^*$. However, $v$ may be `flat' in a certain area, where currently $w \equiv 0$. If this area is large enough, one can play a $1$-Lipschitz $v'\in \mathrm{Alt}^\ep(v)$, which attains value $-\ep$ in that area and equals $v$ in the remaining parts of the domain, yielding objective value $0$. To prevent the $v'$-player from applying this strategy in other admissible parts of the domain, the $w$-player needs to respond with the uniform distribution over all such `flat' areas (properly reweighed, such that none of them is more advantageous than others). 
This gives rise to our second result, the discrete  lower bound. 

\begin{theorem}\label{thm:approx_lower_bound}
    Let $v:[0,1]\rightarrow \RR$ satisfy the assumptions from \Cref{sec:preliminaries}. Then, 
    \begin{equation}
        \E[\tau_{\delta}^{\ep}]
    \geq \frac{\log(1/\delta)}{80\ep^3/L}\sum_{t=1}^{D}\frac{m(B_t)}{8^{D-t}}.
    \end{equation}
\end{theorem}

\Cref{thm:approx_lower_bound} is still instance dependent and holds for any $\kappa_0 > \ep > 0$. When $\ep \rightarrow 0$, one can show that the instance dependence reduces to the \emph{zooming dimension} of $\mu$.

\begin{corollary}\label{cor:approx_lower_bound}
    Let $v:[0,1]\rightarrow \RR$ with zooming dimension $\beta$ satisfy the assumptions from \Cref{sec:preliminaries}. 
    Then, 
    \begin{equation}
        \E[\tau_{\delta}^{\ep}]
    \geq \boldsymbol{\Theta}(\log(1/\delta)\ep^{-2 + \beta}),
    \end{equation}
    when $\ep \rightarrow 0$.
\end{corollary}

Theorem 2 from~\cite{bull2015adaptive} states that for any strategy for tree-armed bandits, the regret $r_T \gtrsim T^{-1/(\beta + 2)}$ at large enough round $T$. When setting $\ep \geq r_T$ and solving for $T$, one can observe that this coincides with~\Cref{cor:approx_lower_bound}. Note however that our result holds for any strategy.


\section{Algorithm}\label{sec:algorithm}

\begin{figure*}
    \centering
    \includegraphics[width=0.8\linewidth]{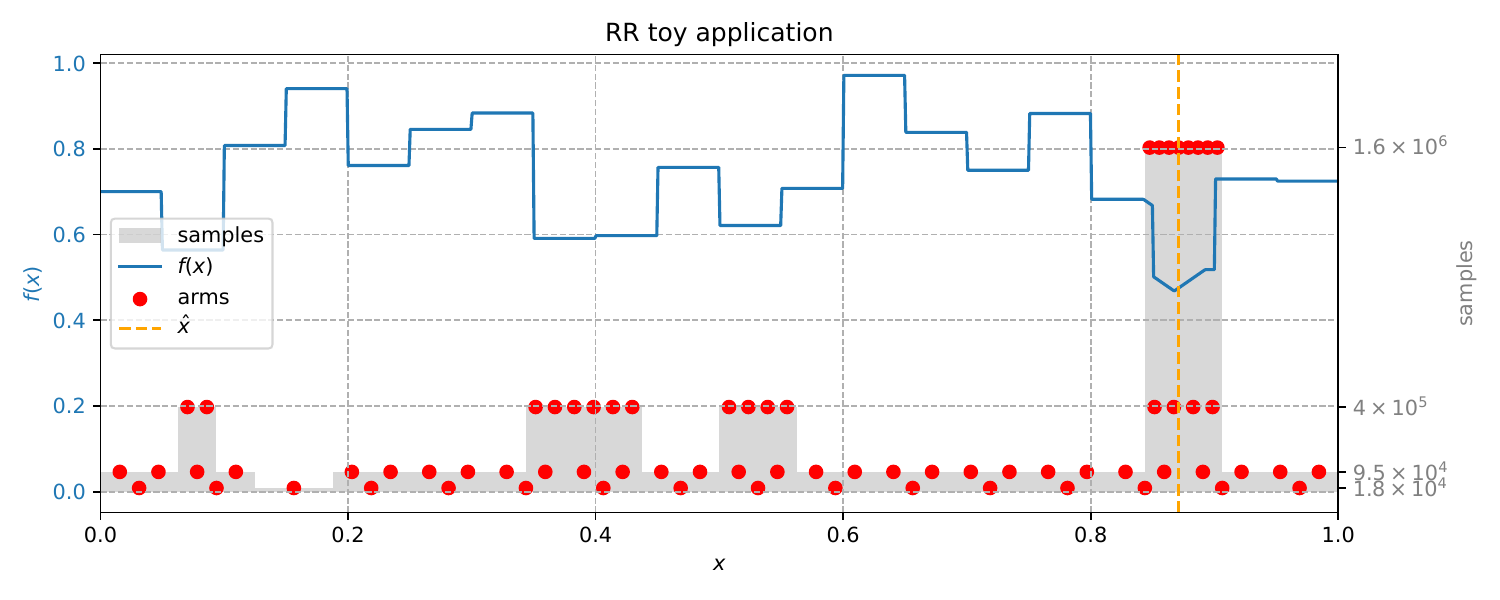}
    \caption{Toy function \textbf{(blue)} and example run of~\Cref{alg:rr} with $D$ rounds. The red dots indicate the \emph{arms} sampled from and the gray area the number of samples for each arm. The minimum estimated by the algorithm $\hat{x}$ is depicted by the orange dashed line.}
    \label{fig:toy_function}
\end{figure*}

In this section, we introduce a simple algorithm which matches the lower bound of the previous section up to a logarithmic factor. We briefly outline the algorithm and provide a convergence analysis. Finally, we outline a simple extension to multi-dimensional domains. The proofs of the theoretical results can be found in~\Cref{sec:proofs_algo}.

\subsection{Algorithm outline}
The algorithm follows the following procedure: In each round $1\leq t \leq D$, it draws sample of fixed size from points $x$ on a uniform grid on $[0,1]$. Then, we construct confidence intervals around $\E[v(x)]$ which, in combination with the Lipschitz property, allow us to exclude $x$ and their neighbourhood, when the estimated mean of $v(x)$ is below some threshold. The number of grid points is doubled after each step, and we only consider points, which do not belong to the neighbourhood of excluded points. We stop when all points, except for an interval of length at most $\ep$, have been excluded. This is illustrated in~\Cref{fig:toy_function}.\\

In the following, denote the set of grid points in round $t$ by
\begin{equation}
    H_{t} = \left\{ \frac{1}{\lceil L \rceil}\left( \frac{k}{2^{t+3}} - \frac{1}{2^{t+4}} \right) \middle| 1 \leq k \leq \lceil L \rceil \cdot 2^{t+3} \right\}.
\end{equation}
Furthermore, let $E_t \subseteq [0,1]$ be the parts of the domain exclude in the end of round $t$ and $G_t$ the parts of the domain, which are left in the beginning of round $t$, such that $G_{t} = G_{t-1} \setminus E_t$. These quantities give rise to~\Cref{alg:rr}. 

\begin{algorithm}
    \SetCustomAlgoRuledWidth{0.45\textwidth}
   \caption{Reject and Refine (RR)}
   \label{alg:rr}
\begin{algorithmic}
   \STATE {\bfseries Input:} Inverted bandit $v(\cdot)$, constants $L, \ep$.
   \STATE Initialize $G_0 = [0,1]$, $t=1$.
   \REPEAT
   \FORALL{$h \in G_{t-1} \cap H_t$ \do}
   \STATE Draw $n_t$ samples from $v(h)$.
   \STATE Compute $\hat{v}(h)$
   \STATE Construct $1 - \frac{\delta}{|H_t|2^t}$ CI of length $\frac{1}{2^{t+3}}$.
   \STATE $a_t^* \leftarrow \mathrm{argmin}_{h \in H_1}\hat{v}(h)$
   \STATE $E_t \leftarrow \bigcup_{h: \hat{v}(h) - \hat{v}(a_t^*) > \frac{12}{2^{t+4}}} \left[h-\frac{1}{2^{t+4}}, h+\frac{1}{2^{t+4}}\right]$
   \STATE $G_{t} \leftarrow G_{t-1} \setminus E_t$
   \STATE $t \leftarrow t+1$
   \ENDFOR   
   \UNTIL{$2^{-t} \leq \ep$}
   \STATE {\bfseries Output:} $a^* = \mathrm{argmin}_{a^*_t}\hat{v}(a_t^*)$.
\end{algorithmic}
\end{algorithm}

The runtime of~\Cref{alg:rr} adheres to the following worst case-guarantee. 

\begin{theorem}\label{thm:rr_upper_bound}
    Let $v:[0,1]\rightarrow \RR$ satisfy the assumptions from \Cref{sec:preliminaries}. Then~\Cref{alg:rr} terminates after $D\triangleq\log_2(1/\ep)$ rounds and the number of samples is bounded by
    \begin{equation}\label{eq:upper_bound}
        \tau_{\delta}^{\ep} \leq 2^{15}L\frac{\log_2(1/\ep)+\log(1/\delta)}{\ep^3}\sum_{t=1}^D \frac{m(B_t)}{8^{D-t}}.
    \end{equation}
\end{theorem}

The bound given in~\Cref{thm:rr_upper_bound} matches the lower bound from~\Cref{thm:approx_lower_bound} up to constant and $\log(\ep^{-1})$ factors, indicating that~\Cref{alg:rr} is nearly optimal. Note that it holds for any $\ep$ irrespective of $v$ being unimodal, allowing the final set $G_D \cap H_{D}$ to contain multiple arms with $v(a) \leq \ep$. This makes~\Cref{alg:rr} fundamentally different from running discrete best arm identification algorithms, as Frank-Wolfe sampling~\cite{wang2021fast}, on a sufficiently fine partition of $[0,1]$. Such a partition may contain multiple best arms, which would hinder methods of this kind from terminating.\\ 

Note that each refinement step comes with a cost: in the best case, one can always exclude half of the remaining domain, so that the number of arms remains constant. In the worst case, the number of arms doubles in each round. Since the confidence intervals of subsequent rounds require four times as many samples, the sample complexity increases by at least a factor four and at most eight. \\

The following result clarifies this
and shows that the asymptotic performance of~\Cref{alg:rr} is at least as good \emph{Adaptive-treed bandits}.

\begin{corollary}\label{cor:algo_complexity}
    Consider $v:[0,1]\rightarrow \RR$, which satisfies the assumptions from \Cref{sec:preliminaries} and has \emph{zooming dimension} $\beta$. Then
    \begin{equation}
        \tau_{\delta}^{\ep} = \mathcal{O}\left((\log{1/\delta} + \log{1/\ep})\ep^{-(\beta + 2)}\right)
    \end{equation}
    and when $\ep \rightarrow 0$, $\E[\tau_{\delta}^{\ep}]$ matches the lower bound from~\Cref{cor:approx_lower_bound}.
\end{corollary}

A key insight from~\Cref{cor:approx_lower_bound} and~\Cref{cor:algo_complexity} is that, in the limit $\ep \rightarrow 0$, continuous best arm identification and regret minimization are nearly equivalent.
\paragraph*{Computational complexity} Computations of~\Cref{alg:rr} besides drawing samples involve estimating the means $\hat{v}(h)$ and their comparison. Since each sample is included in exactly one empirical mean, the respective cost is $\tau_{\delta}^{\ep}$. After each refinement step, the minimal arm is identified, with which the other arms are then compared, amounting to two comparisons per arm present. Clearly, we always draw at least one sample per arm, yielding an overall sample complexity linear in the number of samples, i.e. $\mathcal{O}(\tau_{\delta}^{\ep})$. Hence, our approach is significantly more computationally efficient than Bayesian optimization methods, such as those used in~\cite{nicoli2024physicsinformedbayesianoptimizationvariational}, which require solving an increasingly large linear system after each sample.  
\paragraph*{Application to multi-dimensional parameter optimization}
In order to overcome the limitation of~\Cref{alg:rr} to a $1$-dimensional parameter space, we employ a variety of schemes to map it to higher dimensional parameter spaces. At the heart of each strategy lies the following idea: For some function $f:[0,1]^d \rightarrow \RR$, pick a point $p$ and direction $u$, which give rise to the $1$-dimensional function 
\begin{equation}
g(s) = f((p + us)_{[0,1]^d}),    
\end{equation} reducing the problem to computing 
\begin{equation}
s^* = \mathrm{argmin}_{s\in [0,1]} \ g(s)
\end{equation}
with Lipschitz constant $L \lVert u \rVert_2$.
By subscript $[0,1]^d$, we indicate ``modulo'', i.e. a periodic parameter domain, which applies to the VQA-setting. Depending on the application, the line may be truncated to $[0,1]^d$ instead.\\
Powells method~\cite{powell1964efficient} is a well-established, gradient-free optimization method, which incorporates a scheme of this kind and 
it can be shown that Powell's method efficiently minimizes functions of quadratic form. Typically, Brent's parabolic interpolation method~\cite{brent2013algorithms} is used to find $s^*$, however, any gradient free optimizer can be applied, which makes Powell's method a well-suited extension to~\Cref{alg:rr}.
Another simple extension that falls into this category is to always choose $u_k$ at random. We discuss the respective practical implementations in~\Cref{sec:notes_practical}.

\section{Experiments}\label{sec:experiments}
We compare the empirical convergence of~\Cref{alg:rr} and finite difference based methods on a $1$-d toy example and two VQAs (see~\Cref{fig:results}).  
\begin{figure*}
    \centering
    \hspace{-1.3cm}
    \includegraphics[width=\linewidth]{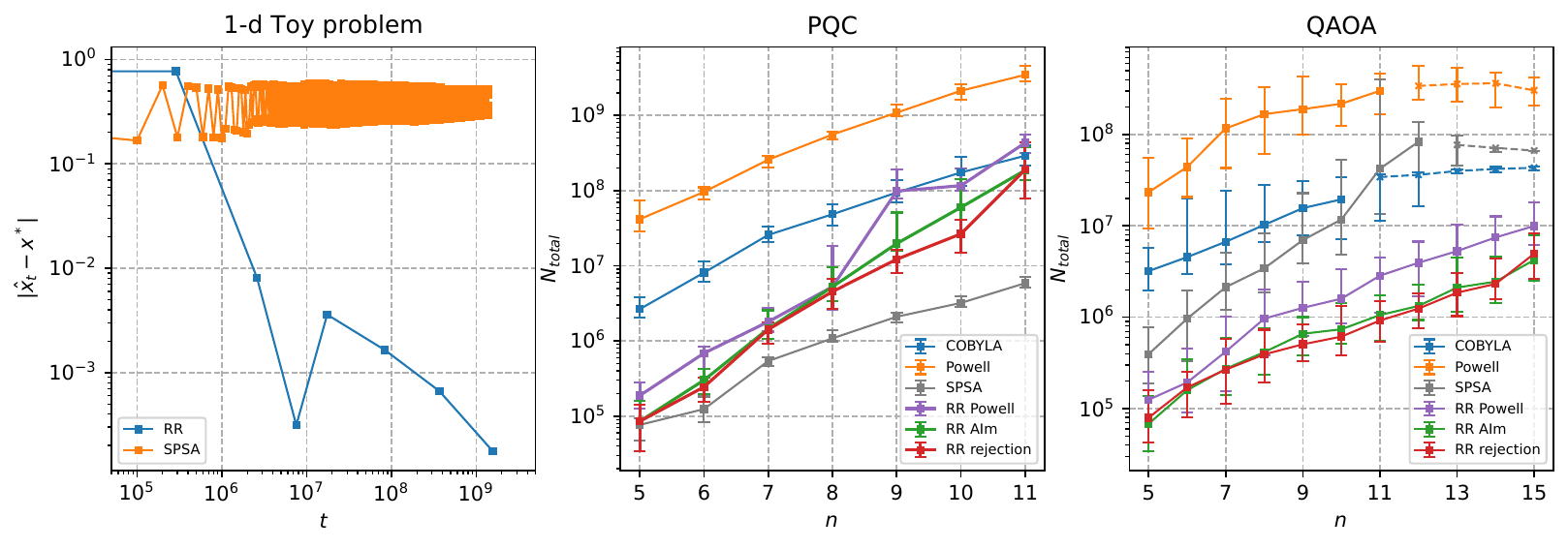}
    \caption{(1): Distance $|\hat{x}_t -x^*|$ to the true optimum at sample $t$ for runs of~\Cref{alg:rr} and SPSA \textbf{(left)}. (2): Median and $(0.25, 0.75)$-quantiles of the sample complexity $N_{\mathrm{total}}$ for optimizing a PQC \textbf{(middle)} and MaxCut-QAOA \textbf{(right)} below the thresholds $C=0.4$ and $C=0.2$, respectively. $N_\mathrm{total}$ refers to the sample complexity and $n$ to the number of qubits. The medians are computed over $20$ and $100$ simulations, respectively. Failure of convergence is indicated by a dashed line.}
    \label{fig:results}
\end{figure*}
Our toy experiment conducted on a function with flat regions, local minima and a linear ``wedge'' around the optimum (see~\Cref{fig:toy_function} and~\Cref{sec:toy}) illustrates how our algorithm profits from global context. As predicted by~\Cref{thm:rr_upper_bound}, the minima $\hat{x}_t$ computed by~\Cref{alg:rr} become consistently closer to the true $x^*$ while SPSA~\cite{spall1998overview} remains stuck in a local minimum when not initialized close to $x^*$. After a constant number of samples,~\Cref{alg:rr} samples exclusively from points near the wedge, indicating that the flat regions no longer influence convergence beyond this point.\\
The second part of our experiments examines the sample complexity of training VQAs to reach a target loss threshold, where we compare COBYLA~\cite{Powell1994}, Powell's method~\cite{powell1964efficient} and SPSA with three implementations of the multidimensional proxy for~\Cref{alg:rr}, outlined in~\Cref{sec:notes_practical}. The first example reproduces the experimental setup from~\cite{Arra21}, which aims to train a PQC with respect to an objective that exhibits barren plateaus. Our second example is a standard application of QAOA~\cite{farhi2014quantumapproximateoptimizationalgorithm} to the MaxCut problem on random graphs, where the approximation ratio was optimized. For a detailed overview over the experimental setup we refer to~\Cref{sec:exp_details}.\\
Although SPSA remains dominant on the PQC example, our algorithms achieve better sample complexity than Powell's method and COBYLA, and effectively outperform all finite difference based optimizers in the QAOA setting. Moreover, Powell's method experiences substantial improvement when equipped with~\Cref{alg:rr} as $1$-d optimization routine. Our experiments show that our simple, unrefined method can compete with state-of-the-art baselines. While we do not consider our multidimensional proxy to constitute a new state-of-the-art approach, our experiments highlight the power of global context and the effectiveness of the bandit-framework we propose, especially in combination with other methods. This may include local methods. One might argue that finite difference strategies could eventually converge when probing enough initial points: bandit algorithms provide a principled strategy for selecting which points to probe.

\section{Conclusion and Outlook}

Here, we introduced a novel approach to VQA optimization based on continuous bandit techniques for exploration. Our work builds on and extends the existing bandit literature in several key ways. First, we generalize best-arm identification to continuous, Lipschitz bandits, formally motivating our choice of alternate set and deriving an instance-specific sample-complexity lower bound. We complement this bound with an algorithm that is provably near-optimal on the unit interval. Moreover, we show that, in the limit, our bounds coincide with those known for regret-minimization methods on an infinite time horizon, a result, which to our knowledge has not been known previously.\\
We also present simple extensions of our algorithm to higher dimensions, enabling us to evaluate the framework on VQA instances that exhibit barren plateaus. Notably, these straightforward schemes achieve competitive sample complexity on the considered instances. Further work is required to develop bandit algorithms that provide more general, non-asymptotic, and practically significant improvement in sample complexity. However, the framework we propose offers a scalable alternative that leverages global information and information-theoretic principles, while remaining inherently robust to noise. Unlike finite-difference–based methods, which become provably infeasible in the presence of barren plateaus.


\Cref{alg:rr} suffers from a relatively large constant cost which could be reduced via more efficient construction of confidence intervals either by further exploiting the Lipschitz property, as~\cite{pmlr-v35-magureanu14} or more effective concentration laws, as e.g.~\cite{ramdas2023betting}. 
On the theoretical side, one natural future direction is the extension of our algorithm to higher dimensions. This comes with the challenge that the algorithm not only needs to choose which parts of the domain it refines, but also along which direction - an extension of this kind for our algorithm requires additional ideas. Moreover, it may be possible to circumvent working with an approximate upper bound and formulate an algorithm in the style of \emph{Track-And-Stop}, which proposes new parameters based on direct approximations of the continuous lower bound.
Finally, our theoretical upper bounds reveal a new insight about the effect of BPs on trainability in $1$-d, which we conjecture to also hold for the multi-dimensional setting: in the asymptotic sense, flat regions only affect the performance of our algorithms when the respective function value is close to the optimal value of the VQA. Therefore, even if BPs cannot be avoided, bandit methods have the potential to mitigate their effects.

\clearpage
\bibliographystyle{plainnat}
\bibliography{sources}

\onecolumn\newpage
\appendix

\section{Proofs of \Cref{sec:lower_bound}}\label{sec:proofs_lb}
In this section, we prove the theoretical results from~\Cref{sec:lower_bound}. We begin with the proof of the continuous lower bound given in~\Cref{thm:exact_lower_bound}.
\begin{proof}[Proof of~\Cref{thm:exact_lower_bound}]
Consider a continuous bandit as defined in~\Cref{sec:preliminaries} with expected reward $\mu(x)$ on domain $\mathcal{X}$ and some policy, which proposes an \emph{action}, i.e. a point $x_t \in \mathcal{X}$ with probability density $\pi(X_t = x_t| Z_t = z_t)$ in round $t$. By $Z_t = (R_{t-1}, X_{t-1}, \dots, R_1, X_1)$, we denote the history of previously observed rewards $R_k$ and actions $X_k$. We assume that the reward distributions of the bandit and the policy have a probability density, which is true in most cases (see e.g. chapter 4.7 in~\cite{lattimore2020bandit}. For simplicity, we further only consider policies with Riemann-integrable densities. One can rewrite the following proof in terms of Lebesgue integrals, however, in the context of $1$-sub-Gaussian, Lipschitz reward distributions, we do not expect any gain from the additional generality. Now consider a bandit with reward distribution $\Lambda(x)$ and expected reward $\lambda \in \mathrm{Alt}^{\ep}(\mu)$. We now proceed similarly as in the proof of Theorem 4.1 in~\cite{kaufmann:tel-03825097}. Let $Z_t^\mu \sim P_\mu(Z_t), Z_t^\lambda \sim P_\lambda(Z_t)$ denote the History of samples under policy $\pi$ and reward distributions $M, \Lambda$, respectively. On the one hand, consider $\mathcal{E}$ to be the event that the policy returns $x_{\tau_\delta^\ep}$, such that $|x_{\tau_\delta^\ep} - x^*| \leq \ep$. If $\pi$ $\delta$-PAC, then we get
\begin{equation}\label{eq:lem_kaufmann}
    \kl(H_{\tau_\delta^\ep}, H_{\tau_\delta^\ep}) \geq \kl(\mathbb{P}_\mu(\mathcal{E}),\mathbb{P}_\lambda(\mathcal{E})) \geq \kl(1-\delta, \delta)
\end{equation}
by Lemma 0.1 in~\cite{kaufmann:tel-03825097}, where we use the same definition for $\mathbb{P}_\mu(\mathcal{E}),\mathbb{P}_\lambda(\mathcal{E})$, i.e. the expected exponents of the log-likelihood ratios. Note that the proof only makes use of the data processing inequality, which does not require a discrete action space. In the following, assume without loss of generality that $\Lambda(x)$ dominates $M(x)$ on all $x \in \mathcal{X}$. On the other hand note that

\begin{equation}
    P_\mu(Z_t) = \prod_{i=1}^t P_\mu(R_i | X_i) \pi(X_i |Z_{i-1})
\end{equation}
and 
\begin{equation}
    P_\lambda(Z_t) = \prod_{i=1}^t P_\lambda(R_i | X_i) \pi(X_i |Z_{i-1})
\end{equation}

and therefore
\begin{align}
    \E_{P_\mu, \boldsymbol{X}}\left[\log{\frac{P_\mu(Z_t)}{P_\lambda(Z_t)}}\right] &= \E_{P_\mu, \boldsymbol{X}}\left[\sum_{i=1}^t \log{\frac{P_\mu(R_i | X_i)}{P_\lambda(R_i | X_i)}}\right]\\
    &=\sum_{i=1}^t \E_{P_\mu, X_i}\left[\log{\frac{P_\mu(R_i | X_i)}{P_\lambda(R_i | X_i)}}\right] \\
    &= \sum_{i=1}^t \int_{\mathcal{X}\times \RR} w_i(x, r) \log{\frac{P_\mu(R_i =r| X_i = x)}{P_\lambda(R_i = r| X_i=x)}}d(x, r),
\end{align}
where $w_i(x,r)$ is the joint probability density of action and reward for $\pi$ in round $i$. Letting $w_i(x)$ the corresponding marginal density of the action, we obtain
\begin{align}
    &= \sum_{i=1}^t \int_{\mathcal{X}\times \RR} w_i(x) \int_{\RR} P_\mu(R_i =r| X_i = x) \log{\frac{P_\mu(R_i =r| X_i = x)}{P_\lambda(R_i = r| X_i=x)}}dr dx\\
    &= \sum_{i=1}^t \int_{\mathcal{X}\times \RR} w_i(x) \kl(M(x), \Lambda(x)) dx\\
    &=\int_{\mathcal{X}\times \RR} \left(\sum_{i=1}^t w_i(x)\right) \kl(M(x), \Lambda(x)) dx.
\end{align}
By~\Cref{eq:lem_kaufmann}, we now get
\begin{align}
    \int_{\mathcal{X}\times \RR} \left(\sum_{i=1}^t w_i(x)\right) \kl(M(x), \Lambda(x)) dx &\geq \kl(1-\delta, \delta)\\
    \Rightarrow t \int_{\mathcal{X}\times \RR} \left(\sum_{i=1}^t \frac{w_i(x)}{t}\right) \kl(M(x), \Lambda(x)) dx &\geq \kl(1-\delta, \delta).
\end{align}
Since the inequality holds for all $\Lambda$ according to our assumption, we can restate this as
\begin{equation}
    t\inf_{\Lambda: \lambda(x) \in \mathrm{Alt}^{\ep}(\mu)} \int_{\mathcal{X}\times \RR} \left(\sum_{i=1}^t \frac{w_i(x)}{t}\right) \kl(M(x), \Lambda(x)) dx \geq \kl(1-\delta, \delta)
\end{equation}    
and since $w_i$ are arbitrary and $w_i \in \mathcal{W}$,
\begin{equation}
    t\sup_{w\in \mathcal{W}}\inf_{\Lambda: \lambda(x) \in \mathrm{Alt}^{\ep}(\mu)} \int_{\mathcal{X}\times \RR} w(x) \kl(M(x), \Lambda(x)) dx \geq \kl(1-\delta, \delta)
\end{equation}

and by $M, \Lambda$ sub-Gaussian,
\begin{equation}
    t\sup_{w\in \mathcal{W}}\inf_{\Lambda: \lambda(x) \in \mathrm{Alt}^{\ep}(\mu)} \int_{\mathcal{X}\times \RR} w(x) \kl(M(x), \Lambda(x)) dx \leq t\sup_{w\in \mathcal{W}}\inf_{\lambda(x) \in \mathrm{Alt}^{\ep}(\mu)}  \int_{\mathcal{X}\times \RR} w(x) \kl(\mu(x), \lambda(x))dx.
\end{equation}
The statement follows directly.
\end{proof}
Note that the above proof works independently of the definition of $\mathrm{Alt}^{\ep}(\mu)$, given that the respective integrals remain well-defined. Hence, one may be interested in the impact of the choice of $\mathrm{Alt}^{\ep}(\mu)$. Indeed, as stated in~\Cref{cor:trivial_lower_bound}, defining 
\begin{equation}
    \mathrm{Alt}^{\ep}(\mu)  \triangleq \{\lambda(x) \, 1\text{-Lipschitz}:\max_x \mu(x) - \max_x \lambda(x) \geq \ep\}
\end{equation}
makes the lower bound much weaker, which we show in the next proof.
\begin{proof}[Proof of~\Cref{cor:trivial_lower_bound}]
    Consider
    \begin{equation}
        \mathrm{Alt}^{\ep}_\eta(\mu)  \triangleq \{\lambda(x) \, 1\text{-Lipschitz}:\max_x \mu(x) - \max_x \lambda(x) \geq \ep + \eta\}
    \end{equation}
    for some $\eta > 0$ and note that $\mathrm{Alt}^{\ep}_\eta(\mu) \overset{\eta \rightarrow 0}{\longrightarrow} \mathrm{Alt}^{\ep}(\mu)$.
    First, we show that $c^*(\mu) \leq \ep$. \\
    By the function class considered in $\mathrm{Alt}^{\ep}(\mu)$, we can choose $\lambda(x)$, such that $(\lambda(x) - \mu(x))^2 \leq (\ep + \eta)^2$ at any point $x$. Therefore,
    \begin{equation}
        c^*(\mu) = \sup_{w \in \mathcal{W}} \inf_{\lambda \in \mathrm{Alt}^{\ep}_\eta(\mu)} \int_{\mathcal{X}} w(x) (\mu(x) - \lambda(x))^2 dx \leq \sup_{w \in \mathcal{W}} \int_{\mathcal{X}} w(x) (\ep + \eta)^2 dx = (\ep + \eta)^2.
    \end{equation}
    For $\eta \rightarrow 0$, the results follows.\\
    Next, we show that $c^*(\mu) \geq \ep$. \\
    If we choose $w(x) = \delta_{x^*}(x)$, which is $1$ if and only if $x=x^*$, we get 
    \begin{equation}
        c^*(\mu) = \sup_{w \in \mathcal{W}} \inf_{\lambda \in \mathrm{Alt}^{\ep}_\eta(\mu)} \int_{\mathcal{X}} w(x) (\mu(x) - \lambda(x))^2 dx \geq \inf_{\lambda \in \mathrm{Alt}^{\ep}_\eta(\mu)} (\mu(x^*) - \lambda(x^*))^2 dx \geq (\ep + \eta)^2,
    \end{equation}
    again by definition of $\mathrm{Alt}^{\ep}_\eta(\mu)$. Letting $\eta \rightarrow 0$ yields the desired result, which concludes the proof.
\end{proof}
Now we proceed to proving~\Cref{thm:approx_lower_bound}. Recall that for convenience, we assume that $\epsilon = 2^{-D}$. Also for convenience, assume that also $\kappa_0$ is an integer power of $2$, such that $\ep = 2^{-S}\kappa_0$ for some $S \in \mathbb{N}$. First, we present the arguments from~\Cref{sec:discrete_lower_bound} in a rigorous way. With slight abuse of notation, let
\begin{equation}
     \mathrm{Alt}^{\ep}(v) \triangleq \{v'(x) \, L\text{-Lipschitz}: \left|\left(\mathrm{argmin}_{x \in [0,1]} v'(x) \right) - x^*\right| > \ep \}.
\end{equation}
In the following, let $A_0 \triangleq v^{-1}[0, \ep)$ and $A_{\leq s} \triangleq \bigcup_{t\leq s} A_t$. Consider some interval $[a, b]$ contained in some level set set $A_t$. If we require $v(a)=v'(a)$ and $v(b)=v'(b)$, $v'$, the Lipschitz-constraint prevents $v'$ from obtaining the value $-\ep$ when the interval is shorter than $\frac{2}{L}(2^t + 2)\ep$. Therefore, we only need to consider intervals $[a,b]\subseteq A_t$ with $b-a \geq \frac{2}{L}(2^t + 2)\ep$.\\
By the assumption in \Cref{sec:preliminaries}, when $2^t \ep \lesssim \kappa_0$, $A_t$ must be the union of exactly two intervals $[a_1, b_1], [a_2, b_2]$. Let $A_t^l$ be the largest interval among the two.
Furthermore, let 
\begin{equation}
    R(v) = \left\{A_t^l : 0 \leq t \leq S \wedge m(A_t^l) \geq \frac{2}{L}(2^t + 2)\ep \right\}
\end{equation}
denote the set of indices with feasible level set. If $R(v) = \emptyset$, $v'$ can only attain some trivial solution, which is given by 
\begin{equation}
    v_0 \triangleq 
    \begin{cases}
    3 \ep + \frac{1}{L}|x - m| & x \in [\frac{m-\ep}{L}, m+\frac{\ep}{L}] \\
    4 \ep & x \in A_{\leq 2} \setminus [\frac{m-\ep}{L}, m+\frac{\ep}{L}]\\
    v(x) & \mathrm{otherwise},
    \end{cases}
\end{equation}
where $m \in A_{\leq 2}$ is chosen, such that it has distance $\ep$ from its boundary, i.e. $\inf_{x \in \partial A_{\leq 2}} \lVert m - x \rVert_2 = \frac{\ep}{L}$. This function equals $v$ on all higher level sets, and is constant on $A_{\leq 2}$ except for a small bump, which touches the boundary of $A_{\leq 2}$ and represents a reasonable ``approximate'' infimum with respect to $v'$ when the latter can not have a minimum outside of $[x^* - \frac{\ep}{L}, x^* + \frac{\ep}{L}]$ with value smaller than $0$. \\
In the following, we restrict $\mathrm{Alt}^{\ep}(v)$ to a set of functions, which equal $v$ everywhere except on some $A_t^l\in R(v)$, where they have one wedge, i.e. of the form $-\ep \frac{1}{L}|x - m|$ for some $m$. 
\begin{lemma}\label{lem:first_approx_bound}
    \begin{equation}\label{eq:rhs_lemma}
    \frac{2}{L}\sup_{w \in \mathcal{W}} \inf_{v' \in \mathrm{Alt}^{\ep}(v)} \int_{0}^{1} w(x)(v'(x)-v(x))^2dx \leq  \frac{2\ep^3}{\sum_{A_t \in R(v)} \frac{m(A_t)}{(2^t+2)^3}}.
\end{equation}
\end{lemma}
\begin{proof}
If $R(v) = \emptyset$, \Cref{eq:rhs_lemma} is trivially satisfied.
Otherwise, we restrict $v'$ to have a wedge on some $A_t^l \in R(v)$, which attains $v'(x^*(v'))=-\ep$ is equal to $v$ everywhere else. We denote the restricted alternate set by $\overline{\mathrm{Alt}}^{\ep}(v)$. Observe that $(v(x) - v'(x))^2 \leq (2^t + 2)^2 \ep^2$, which it does on a set $F_t \subseteq A_t$ of length $m(F_t) \leq  \frac{2}{L}(2^t + 2)\ep $. Therefore, 
\begin{align}
    \sup_{w \in \mathcal{W}} \inf_{v' \in \overline{\mathrm{Alt}}^{\ep}(v))} \int_{0}^{1} w(x)(v'(x)-v(x))^2dx &\leq  \sup_{w \in \mathcal{W}} \min_t \inf_{\substack{F_t \subseteq A_t^l\\ m(F_t) = \frac{2}{L}(2^t + 2)\ep}}  \int_{A_t^l} w(x)(2^t + 2)^2 \ep^2 \mathbbm{1}_{F_t} dx\\
    &\leq \sup_{w \in \mathcal{W}} \min_t \inf_{\substack{F_t \subseteq A_t^l\\ m(F_t) = \frac{2}{L}(2^t + 2)\ep}} (2^t + 2)^2 \ep^2 \int_{F_t}  w(x) dx . \label{eq:pf_lem1}
\end{align}
For $w$ to maximize~\Cref{eq:pf_lem1}, $w$ needs to be uniform over $A_t^l$, i.e. $\frac{w_t}{m(A_t^l)}$ for constants $w_t > 0$. Plugging this in yields
\begin{equation}
    \sup_{w \in \mathcal{W}} \min_t \inf_{\substack{F_t \subseteq A_t^l\\ m(F_t) = \frac{2}{L}(2^t + 2)\ep}}  (2^t + 2)^2 \ep^2 \int_{F_t}  w(x) dx \leq \sup_{ \substack{{w_t}_{A_t\in R(v)} \\ \sum_t w_t = 1} } \min_t \frac{2}{L}(2^t + 2)^3\ep^3\frac{w_t}{m(A_t^l)}.
\end{equation}
To minimize this expression, we need to choose $w_t$, such that $\frac{2}{L}(2^t + 2)^3\ep^3\frac{w_t}{m(A_t^l)}$ is the same for all $t$.
Hence, the optimal $w$ for the upper bound on $c^*(v)$ reads
\begin{equation}\label{eq:lamda_multid}
    w^*_s(x) = \begin{cases}
        \frac{m(A^l_s)/(2^s + 2)^{3}}{\sum_{A_t \in R(v)} m(A^l_t)/(2^t + 2)^{3}} & x \in A_s^l \cap R(v)\\
        0 & \mathrm{otherwise}.
    \end{cases} 
\end{equation}

Note that the factors $\frac{2}{L}$ cancel out. Plugging this into the above expression, we get
\begin{align}
    \sup_{w \in \mathcal{W}} \inf_{v' \in \mathrm{Alt}^{\ep}(v)} \int_{0}^{1} w(x)(v'(x)-v(x))^2dx &\leq \sup_{w \in \mathcal{W}} \int_{0}^{1}w(x)(v'(x)-v(x))^2dx \\
    &\leq \frac{\ep^3}{\sum_{A_t \in R(v)} \frac{m(A_t^l)}{(2^t+2)^3}} \\
    &\leq \frac{2\ep^3}{\sum_{A_t \in R(v)} \frac{m(A_t)}{(2^t+2)^3}},
\end{align}
where the last inequality follows from $m(A_t) \leq 2 m(A_t^l)$ for $A_t \in R(v)$.
\end{proof}

\begin{lemma}\label{lem:1d_assum}
    Let $v:[0,1]\rightarrow \RR$ satisfy the assumptions from \Cref{sec:preliminaries}. Then, 
    \begin{equation}\label{eq:assumption_lb}
        \sum_{t=1}^S \frac{m(A_t)}{(2^t + 2)^{3}}  \leq \max\left(2\sum_{A_t \in R(v)} \frac{m(A_t^l)}{(2^t + 2)^{3}}, \frac{\ep}{4}\right).
    \end{equation}
\end{lemma}
\begin{proof}
    If the right hand side of~\Cref{eq:rhs_lemma} is less than $16\ep^2$, the $v^*$-player takes $v^*=v'$ and otherwise she takes $v^*=v_0$.

In case $R(v)$ is empty, the $v^*$-player sets $v^*=v_0$.

Note now that
\begin{equation}
    \sum_{A_t \not\in R(v), t \leq S} \frac{m(A_t)}{(2^t+2)^3} \leq \frac{\ep}{(2^t+2)^2} < \frac{\ep}{8}.
\end{equation}
 
and if the right hand side of~\Cref{eq:rhs_lemma} is less than $16\ep^2$,
\begin{equation}
    \sum_{A_t \in R(v)} \frac{m(A_t)}{(2^t+2)^3} > \frac{\ep}{8}.
\end{equation}

Hence whenever $v^* \neq v_0$,
\begin{equation}
    \sum_{t=1}^{S} \frac{m(A_t)}{(2^t+2)^3} < 2\sum_{k \in R(v)} \frac{m(A_t)}{(2^t+2)^3}
\end{equation}

On the other hand if the right hand side is at least $16\ep^2$, then 
\begin{equation}
    \sum_{A_t \in R(v)} \frac{m(A_t)}{(2^t+2)^3} \leq \frac{\ep}{8}
\end{equation}
and so
\begin{equation}
    \sum_{t=1}^{S} \frac{m(A_t)}{(2^t+2)^3} < \frac{\ep}{4}.
\end{equation}
\end{proof}

\begin{corollary}\label{cor:lem_ass_1d}
    Let $v:[0,1]\rightarrow \RR$ satisfy the assumptions from \Cref{sec:preliminaries}. Then
    \begin{equation}
        \sup_{w \in \mathcal{W}} \int_{0}^{1}w(x)(v'(x)-v(x))^2dx \leq \frac{4\ep^3}{\sum_{t=1}^{S} \frac{m(A_t)}{(2^t+2)^3}}.
    \end{equation}
\end{corollary}
\begin{proof}
    By~\Cref{lem:1d_assum}, we have that either
    \begin{equation}
        \sum_{t=1}^S \frac{m(A_t)}{(2^t + 2)^{3}}  \leq 2\sum_{A_t \in R(v)} \frac{m(A_t^l)}{(2^t + 2)^{3}}
    \end{equation}
    or
    \begin{equation}
        \sum_{t=1}^S \frac{m(A_t)}{(2^t + 2)^{3}}  \leq \frac{\ep}{4}.
    \end{equation}
    In the former case, the statement follows directly from~\Cref{lem:first_approx_bound}.
    In the latter case, we get
    \begin{equation}
        16\ep^2 = \frac{4\ep^3}{\ep/4} \leq \frac{4\ep^3}{\sum_{t=1}^{S} \frac{m(A_t)}{(2^t+2)^3}},
    \end{equation}
    and again, the desired statement follows from~\Cref{lem:first_approx_bound}.
\end{proof}

\begin{proof}[Proof of~\Cref{thm:approx_lower_bound}]
By~\Cref{cor:lem_ass_1d}
\begin{equation}
    \frac{2}{L}c^*_\ep(v) \leq \sup_{w \in \mathcal{W}} \int_{0}^{1}w(x)(v'(x)-v(x))^2dx \leq \frac{4\ep^3}{\sum_{t=1}^{S} \frac{m(A_t)}{(2^t+2)^3}}.
\end{equation}

Since $\kappa_0$ only depends on $v$ and not $\ep$, we then have for $\ep$ sufficiently small that
\begin{equation}
    \frac{2}{L}c^*_\ep(v) \leq \sup_{w \in \mathcal{W}} \int_{0}^{1}w(x)(v^*(x)-v(x))^2 dx \leq \frac{5\ep^3}{\sum_{t=1}^{T} \frac{m(A_t)}{(2^t+2)^3}}.
\end{equation}

This finally leads to

\begin{equation}
    \E[\tau_{\delta}^{\ep}] \geq \frac{\log(1/\delta)}{10\ep^3/L}\sum_{t=1}^{D} m(A_t)/(2^t+2)^3
\geq \frac{\log(1/\delta)}{80\ep^3/L}\sum_{t=1}^{D} \frac{m(A_t)}{8^t}.
\end{equation}
\end{proof}

\begin{proof}[Proof of~\Cref{cor:approx_lower_bound}]
    Let $C, \beta$, such that we require $Cr^{-\beta}$ sets of diameter $r/8$ to cover the sets $X_r = \{x \in [0,1]: r \leq \mu(x^*) - \mu(x) \leq r\}$ for any $r \in [0,1/2]$. Furthermore, let $\eta(\ep)$ be the largest $0 < \eta < \ep$, such that $N_{\eta}(X_{\eta}) \geq C'r^{-\beta}$. By $C'$, we denote the smallest constant, such that $\eta(\ep)$ exists for all $\kappa_0>\ep > 0.$ \\
    Note that $C'$, must exist and be independent of $\ep$. Otherwise, we could construct constant $C''$ and $\beta' < \beta$, such that we can cover all $X_r$ with at most $C'' r^{-\beta'}$ sets of diameter $r/8$, which contradicts the definition of the zooming dimension.\\
    Now let $D(\ep) = \lfloor \log_2{\ep} \rfloor$. Then, by unimodality of $v$, $B_{t-1} \cup B_t$ consists of at most two intervals and therefore
    \begin{equation}
        2^t N_{2^t}(B_{t-1} \cup B_t) \leq 2\cdot (B_{t-1} \cup B_t).
    \end{equation}
    Furthermore, note that
    \begin{equation}
        X_r \subset B_{t-1} \cup B_t \quad \forall \ 2^{-t} \leq r \leq 2^{-(t-1)}.
    \end{equation}
    Thus, for $\ep < \kappa_0$,
    \begin{align}\label{eq:cor_approx_lb1}
        \sum_{t=1}^{D(\ep)} 8^t m(B_t) &\geq \frac{1}{2\cdot 8} \sum_{t=2}^{D(\ep)} 8^t m(B_{t-1} \cup B_t)\\
        &\geq \frac{4^{D(\ep)}}{32} N_{2^{D(\ep)}}(B_{{D(\ep)}-1} \cup B_{D(\ep)}) \\
        &\geq \frac{4^{D(\ep)}}{32}\sup_{2^{-D(\ep)} \leq r \leq 2^{-(D(\ep)-1)}} N_r(X_r).
    \end{align}
    Since $\eta(\ep)=\ep$ in the limit, we can conclude that 
    \begin{align}
        \lim_{\ep \rightarrow 0} \sup_{2^{-D(\ep)} \leq r \leq 2^{-(D(\ep)-1)}} N_r(X_r) &= \lim_{\ep \rightarrow 0} \sup_{2^{-D(\ep)} \leq r \leq 2^{-(D(\ep)-1)}} N_r(X_r)\\
        &\geq \lim_{\ep \rightarrow 0} C' 2^{\beta D(\eta(\ep))}\\
        &= \lim_{\ep \rightarrow 0} \ep^{-\beta}.
    \end{align}
    When plugging this back into~\Cref{eq:cor_approx_lb1}, the statement follows directly.
\end{proof}

\section{Proofs of~\Cref{sec:algorithm}}\label{sec:proofs_algo}
For illustration purposes, we start by analyzing~\Cref{alg:rr} for $L=1$. Let $G_0 \triangleq [0,1]$. The goal for this round is to exclude all points $x$, for which $v(x) \geq \frac{1}{2}$, i.e. obtain $E_1 \supseteq B_1$, with probability at least $1 - \frac{\delta}{2^1}$. 
Therefore, we pull each of the $2^{1+3} = 16$ arms in $H_1$ sufficiently many times to create symmetric confidence intervals for $v(h)$, $h \in H_1$, of length at most $1/2^{1+3}=1/16$, i.e.\ of the form 
\begin{equation*}
    v(h) = \hat{v}(h) \pm \frac{1}{2^{1+4}} = \hat{v}(h) \pm \frac{1}{32} =:[\underline{v}(h),\overline{v}(h)],
\end{equation*}
at confidence level $1-\delta/2^{2 \cdot 1+3}=1-\delta/32$, making the multiple confidence level at least $1-\delta/2$. When constructing these confidence intervals, we do this independently for each $h \in H_1$, i.e.\ without using any structure of $v$ to make inference from pulls of nearby arms. Let $E_1$ be the set of points $x \in G_0$ for which we can conclude after these pulls that $v(x)>9/32$, given that all confidence intervals are correct and let $B_{\geq t}$ for $\cup_{s \geq t}B_s$.\\
Assuming that all the confidence intervals indeed cover the true value $v(h)$, the Lipschitz property tells us that the true function value $v(a^*)$ of the arm $a^*$ with the smallest empirical mean, i.e.\ $a^*=\mathrm{argmin}_{h \in H_1}\hat{v}(h)$, is at most $1/32$. This implies that 
\begin{equation*}
    \overline{v}(a_1) \leq \frac{1}{32}+\frac{1}{16}=\frac{3}{32}.
\end{equation*}
Suppose now that $x \in B_1$, i.e.\ that $v(x)>1/2$. Then there is an $h \in H_1$ at distance at most $1/32$ from $x$, which by the Lipschitz property must satisfy $v(h)>15/32$ and hence the lower end of the confidence interval for $v(h)$ satisfies 
\begin{equation*}
    \underline{v}(h)>\frac{13}{32},
\end{equation*}
from which we can conclude that $v(x) > 12/32$ for all $x \in [h-1/32, h+1/32]$. Taken together, the above allows us to drop the simplifying assumption that $v(x^*)=0$ and leads to the conclusion
\begin{equation*}
    v(x)-v(a^*)>\frac{9}{32} \qquad \forall x \in [h-1/32, h+1/32].
\end{equation*} 
In particular, this is the case when $\hat{v}(h) - \hat{v}(a^*) > \frac{12}{32}$.
Hence $x \in E_1$ and since $x$ was arbitrary, $E_1 \supseteq B_{\geq 2}$ with probability at least $1 - \frac{\delta}{2}$.
For $G_1=G_0 \setminus E_1$, we immediately get that $G_1 \subseteq B_{\geq 2}$. Moreover, as $1/2^2 < 9/32$, we get  $B_{\geq 3} \subseteq G_1 \subseteq B_{\geq 2}$.\\
Continuing this procedure for $t=2, \dots, D$ with confidence levels $1 - \frac{\delta}{2^t}$ yields~\Cref{alg:rr}.

\begin{proof}[Proof of~\Cref{thm:rr_upper_bound}]
    
The algorithm works as follows for rounds $t=2,3,\ldots, D$. 
Pull each of the arms of $H_t \cap G_{t-1}$ sufficiently many times to obtain symmetric confidence intervals 
\begin{equation}
    v(h) = \hat{v}(h) \pm \frac{1}{2^{t+4}}
\end{equation}
at confidence level $1-\delta/(2^{t}|H_t|$. Let $E_t$ be the set of points in $G_{t-1}$, for which can conclude that $v(x)-v(a_t) > 9/2^{t+4}$, provided that all confidence intervals are correct, and set $G_{t}=G_{t-1} \setminus E_t$. This yields the sets $G_2,G_3,\ldots, G_{D}$. By the union bound, this results in a multiple confidence level of at least $1-\delta/2^t$.\\

Recall that for constructing a symmetric confidence interval on confidence level 1-$\alpha$ of length $2\ell$ for the mean of a Gaussian distribution with unit variance, it suffices that the sample size $N$ satisfies
\begin{equation}
    N \geq \frac{2\log(2/\alpha)}{\ell^2}.
\end{equation}
Hence, the number of pulls per arm is
$2^{2t+9}\log(2^{2t+4}/\delta)$ and so the total number of pulls in this round becomes
\begin{equation}
    2^{2t+9}|H_t|\mu(G_{t-1})\log(2^{t}|H_t|/\delta) = L 2^{3t+12}\mu(G_{t-1})\log(L2^{2t+4}/\delta).
\end{equation}
Let 
\begin{equation}
    a_t=\mathrm{argmin}_{h \in H_t}v(h).
\end{equation}
Then $v(a_t) \leq 1/2^{t+4}$ and thus 
\begin{equation}
    \overline{v}(a_t) \leq \frac{3}{2^{t+4}}.
\end{equation}
Consider an $x \in G_{t-1}$ for which $v(x) > 1/2^t$, if such $x$ exists. Then there exists and $h \in H_t$ such that $v(h)> 15/2^{t+4}$. If $h \in G_{t-1}$, so that the $h$-arm is actually pulled in this round, this gives 
\begin{equation}
    \underline{v}(h) > \frac{13}{2^{t+4}},
\end{equation}
from which we can conclude that $v(x) > 12/2^{t+4}$ and hence $v(x)-v(a_t)>9/2^{t+4}$ and thus $x \in E_t$ so that $x \not \in G_t$.
If on the other hand $h \not \in G_{t-1}$, we already know from the previous round that $v(h)-v(a_{t-1}) > 9/2^{t+3}$ and so $v(h)-v(a_t)>9/2^{t+3}$ and we can conclude that $v(x)-v(a_t) > 9/2^{t+3}-2^{t+4} > 9/2^{t+4}$ and thus again $x \in E_t$. Thus provided that all confidence intervals are correct, $B_{\geq t+2} \subseteq G_{t} \subseteq B_{\geq t+1}$.

\smallskip

Now run the above for $t=1,\ldots,D$. After this, we have $B_{\geq D+2} \subseteq G_{D} \subseteq B_{ \geq D+1}$, i.e.\ $v^{-1}[0,\ep/2] \subseteq G_{D} \cap v^{-1}[0,\ep]$. Hence $G_{D}$ is non-empty and all elements $x \in G$ are $\ep$-optimal arms, so any arm in $G_{D} \cap H_D$ is $\ep$-optimal.

The complexity of the algorithm is upper bounded by
\begin{equation}
    \sum_{t=0}^{D} L 2^{3t+12}\mu(G_{t-1})\log(L2^{2t+4}/\delta) \leq 2^{14} L \sum_{t=1}^{D} 8^t \mu(B_{\geq t}) (t+\log(/\delta)).
\end{equation}
Since $\sum_{j=1}^t j \, 8^j < 2t \cdot 8^t$ and $\sum_{j=1}^t 8^j < 2 \cdot 8^t$, the right hand side is bounded by
\begin{equation}
    2^{15} L\sum_{t=1}^{D} (t+\log(1/\delta))8^t \mu(B_{t}) < 2^{15}L({D}+\log(1/\delta)) \sum_{t=1}^{D} 8^t \mu(B_{t}).
\end{equation}
The expression on the right equals
\begin{equation}
    2^{15}L\frac{\log_2(1/\ep)+\log(1/\delta)}{\ep^3}\sum_{t=1}^{D} \mu(B_t)/8^{D-t}.
\end{equation}
So, in summary, the above defines an algorithm $\mathcal{A}$ with
\begin{equation}
    \tau_{\mathcal{A}} \leq 2^{15}L\frac{\log_2(1/\ep)+\log(1/\delta)}{\ep^3}\sum_{t=1}^{D} \frac{\mu(B_t)}{8^{D-t}}
= 2^{15}L(D+\log(1/\delta))\sum_{t=1}^{D} 8^t \mu(B_t).
\end{equation}
\end{proof}

Finally, we present the proof of~\Cref{cor:algo_complexity}.

\begin{proof}[Proof of~\Cref{cor:algo_complexity}]
    For convenience, we rewrite
    \begin{equation}
        \frac{1}{\ep^3}\sum_{t=1}^D \frac{m(B_t)}{8^{D-t}} = \sum_{t=1}^D 8^t m(B_t).
    \end{equation}
    For the number $N_{2^{-t}/8}(B_t)$ of sets with radius $2^{-t}/8$ required to cover $B_t$, it holds that $\frac{m(B_t)}{2^{-t}/4} \leq N_{2^{-t}/8}(B_t) \leq C 2^{-\beta t}$ and hence $m(B_t) \leq C 2^{t(\beta-1)}$ for some constant $C$. Using the formula for finite geometric sums and $\ep = 2^{-D}$, we get
    \begin{equation}
        \sum_{t=1}^D 8^t m(B_t) \leq \sum_{t=1}^D C 2^{(2+\beta)t} \leq C 2^{2+\beta)D} = \mathcal{O}(\ep^{-2+\beta}).
    \end{equation}
    The proof of the second part can be done analogously to the proof of~\Cref{cor:approx_lower_bound}.
\end{proof}

\section{Details on the numerical experiments}
In this section, we give a detailed overview over the implementation and experimental setup used in~\Cref{sec:experiments}.
\subsection{Notes on practical implementation}\label{sec:notes_practical}
The extensions of~\Cref{alg:rr} to higher dimensions we consider may require a large number of $1$-d optimizations. To control the sample complexity, we restrict the maximal depth of~\Cref{alg:rr} to a a threshold $D_{\mathrm{max}}$. Since most lines $g(s)$ may not contain the optimum, we terminate the procedure at depth $1$ when the values observed so far suggest that significant improvement at higher depth is unlikely. The resulting computation yields $g_k(\hat{s}_k^*) \leq g_k(s_k^*) + \ep$ in step $k$, where 
\begin{equation*}
    \hat{s}_k = \underset{x\in \boldsymbol{H}_k}{\mathrm{argmin}} \ g_k(x) \quad \text{and} \quad \boldsymbol{H}_k = \bigcup_{t=1}^{D_{\mathrm{max}}} H_t.
\end{equation*}
Observe that $g_k(0) = \hat{g}_{k-1}(\hat{s}^*_{k-1})$ in steps $k > 1$, yielding  $s_k = 0$ as additional observation. \\
When using~\Cref{alg:rr} in combination with Powell's method (RR Powell), we incorporate this by selecting the minimizer of $\boldsymbol{H}_k \cup \{0\}$ as $\hat{s}^*_k$: we only \emph{accept} $\hat{s}_k$ when it improves on the current minimum and choose $\hat{s}^*_k = 0$ otherwise, such that $p_{k+1} = p_k$. In the approach, where sample $u_k$ uniformly at random, we update $p_{k+1}$ to $p_k + \hat{s}_k u_k$
\begin{enumerate}[label=(\roman*)]
    \item with some acceptance probability $a_k$, depending on $g_k(\hat{s}_k)$ and $g_k(0)$ (reject)
    \item if $g_t(t^*_k) < g_{t-1}(t^*_k)$ (AIm),
\end{enumerate}
and $p_{k+1} = p_k$ otherwise. We consider $a_k = e^{-q(\hat{s}^*_{k}-\hat{s}^*_{k-1})}$ a reasonable acceptance probability. Finally, $\delta$ may also be chosen in a favourable way, since the practical performance may suffer from overly conservative multiple confidence estimates. In summary, this results in $q$, $D_{max}$, $\delta$ and $L$ as hyperparameters.

\subsection{Details on Toy problem}\label{sec:toy}

Our toy example, illustrated in~\Cref{fig:toy_function}, is a piecewise constant version of the function 
\begin{equation*}
    f(x)=1 - ((\sin(13x)\sin(27x) + 1)/4)
\end{equation*} 
of the form 
\begin{equation}
    \tilde{f}(x) = f(i/20) \qquad \mathrm{if} \quad \frac{i}{20} -\frac{1}{40} \leq x < \frac{i}{20} + \frac{1}{40}, \quad x=1,\dots,20.
\end{equation} 
and a wedge with slope $2$ around the minimum of $f$, which is at $x^* \approx 0.8675$. Note that this wedge causes $f$ to have zooming dimension $0$, which enables the convergence depicted in~\Cref{fig:results}. In the case of SPSA we evaluated the empirical mean of $10^5$ samples at each evaluation in order to obtain a comparable number of samples. For illustration purposes, it was initialized at $x=0.5$ and remained stuck in that region.\\
This experiments was conducted on a laptop and took roughly $5$ minutes.

\subsection{Details on experimental setups for VQAs}\label{sec:exp_details}
\paragraph*{Implementation of PQC}
In the example from~\citet{Arra21}, the goal is to train the parameters in an $n$-qubit circuit $V(\theta)$, such that $V(\theta) \ket{0}= \ket{0}$, with the all-zero state and optimizing following local cost function 
\begin{equation}
    C(\theta) = \tr[O_L V(\theta) \ket{0}\bra{0} V^T(\theta)],
\end{equation}
where
\begin{equation}
    O_L = \mathbbm{1} - \frac{1}{n}\sum_{i=1}^n \ket{0_i}\bra{0_i}.
\end{equation}
We trained the same PQC as in~\cite{Arra21} for $n=5, \dots, 11$ qubits with $p=n$ layers, ensuring that the objective function exhibits barren plateaus. For each $n$, we conducted $20$ simulations with initial parameters sampled uniformly at random. The optimization was terminated when either a threshold of $C=0.4$ was reached or the respective algorithm converged. For all versions of RR, we found $q=400$, $D_{\max} = 1$, $\delta=20$ and $L=0.5$ to be optimal. We tested COBYLA and Powell's method with the default settings from the respective SciPy functions and used $N=10^3, 10^4, 10^5$ circuit evaluations to approximate the objective in each iteration. The latter two only converged for $N=10^5$ on all examples. SPSA was implemented using \textit{pennylane} with default settings. We found $N=10^3$ to work best.\\
\paragraph*{Implementation of QAOA}
Our second example demonstrates the application of (vanilla) QAOA~\cite{farhi2014quantumapproximateoptimizationalgorithm} to the MaxCut problem on graphs with $n=5, \dots, 15$ vertices. The graphs were at random according to the Erd\H{o}s-Rényi model~\cite{erdHos1961strength} with edge probability $0.5$. As objective, we used $1-R_a$, where $R_a$ denotes the approximation ratio. For each $n$, we performed $100$ simulations, each utilizing a randomly generated graph and initial parameters sampled uniformly at random. For RR, we used the same hyperparameters as in the previous experiments. For the other methods, we tested $N=10^3, 10^4, 10^5, 10^6$ and found that even for $N=10^6$, less than half of them converged below the threshold of $C=0.2$ for $n>11$. Their success rate of optimization varied between $1\%$ and $30\%$ for $n=15$ across all tested values of $N$. In ~\Cref{fig:results}, we depict runs with $N=10^5$ for COBYLA and Powell's method and $N=10^4$ for SPSA, which were the smallest number of shots for which we observed reasonable convergence rates. We terminated the SPSA experiments when they roughly exceeded $10^8$ samples. The $40\%, 25\%$ and $10\%$ of the depicted runs for $n=13,14,15$, which did reach the threshold $C=0.2$ did so after at least $3\times 10^7$ samples.\\
\paragraph*{Further notes}
The simulations of both experiments were carried out separately for each seed on Intel Xeon Gold 6130 cpus with 16 CPU cores. Each simulation included all system sizes $n$ and took between $4$ and $16$ hours. 
Finally, it is worth noting that the sample complexity of these two methods is significantly smaller than that reported by~\cite{Arra21}, which our algorithm outperforms by several orders of magnitude. A possible explanation for this discrepancy is that we employed a different number of shots to approximate the objective.

\end{document}